%% file: InertialLMIs.tex
\documentclass[letterpaper, 10 pt, journal, twoside]{IEEEtran}
\IEEEoverridecommandlockouts
\include{preamble}

\usepackage{graphicx}
\usepackage{empheq}
\usepackage{amsthm}
\usepackage{bbding}
\usepackage{nicefrac}

\newtheorem{theorem}{Theorem}
\newtheorem{definition}{Definition}
\newtheorem{proposition}{Proposition}
\newtheorem{remark}{Remark}
\newtheorem{corollary}{Corollary}
\usepackage{qting}
\usepackage{hyperref}

\newcommand{\revone}[1]{{\color{black}#1}}
\newcommand{\SO}[1]{\mathsf{SO}(#1)}
\renewcommand{\SE}[1]{\mathsf{SE}(#1)}
\newcommand{\so}[1]{\mathsf{so}(#1)}
\renewcommand{\se}[1]{\mathsf{se}(#1)}

\newcommand{\vel}{\mbox{\boldmath $v$}}

\title{\q{title}{\revone{Linear Matrix Inequalities for Physically-Consistent \revone{Inertial} Parameter Identification: A Statistical Perspective on the Mass Distribution}}}

\author{Patrick~M.~Wensing$^{1}$, Sangbae Kim$^{2}$, and Jean-Jacques E. Slotine$^{2}$ 
\thanks{$^{1}$Patrick M. Wensing is with the Department of Aerospace and Mechanical Engineering,
        University of Notre Dame, Notre Dame, IN 46556 USA 
{\tt\small  pwensing@nd.edu}}%
\thanks{$^{2} $Sangbae Kim and Jean-Jacques E. Slotine are with the Department of Mechanical Engineering,
         MIT, Cambridge, MA 02139 USA
{\tt\small \{sangbae,jjs\}@mit.edu}}% 

\thanks{Digital Object Identifier (DOI): see top of this page.} \vspace{-15px}}

\begin{document}

\maketitle
~\\[-4ex]
%This note shows that the set of physically-realizable 6D rigid-body inertia tensors is a convex cone that is linear matrix inequality (LMI) representable. 
\begin{abstract}

With the increased application of model-based whole-body control in legged robots, there has been a resurgence of research interest into methods for accurate system identification. An important class of methods focuses on the {\em \revone{inertial} parameters} of rigid-body systems. These parameters consist of the mass, first mass moment (related to center of mass location), and rotational inertia matrix of each link. The main contribution of this paper is to formulate physical-consistency constraints on these parameters as Linear Matrix Inequalities (LMIs). The use of these constraints in identification can accelerate convergence and increase robustness to noisy data. It is critically observed that the proposed LMIs are expressed in terms of the covariance of the mass distribution, rather than its rotational moments of inertia. With this perspective, connections to the classical problem of moments in mathematics are shown to yield new bounding-volume constraints on the mass distribution of each link. While previous work ensured physical plausibility or used convex optimization in identification, the LMIs here uniquely enable both advantages. Constraints are applied to identification of a leg for the MIT Cheetah 3 robot. Detailed properties of transmission components are identified alongside link inertias, with parameter optimization carried out to  global optimality through semidefinite programming. %resulting in an RMS torque estimation error of just over 2Nm on each joint.


%More complete conditions on physical realizability surprisingly result in an LMI of lower dimension in comparison to previous work. As an extension, we show how the the statistical theory of moments can be used to enforce bounding-ellipsoid constraints on the spatial extent of each rigid body in the system.

% It is known that the standard manipulator equations of a rigid-body system are linear in these \revone{inertial} parameters, enabling many methods for offline model identification and online adaptive control. Recent work has shown that previous  methods can lead to parameter estimates that are non-physical (i.e. can not possibly represent any physical rigid-body system), due to the violation of certain triangle inequalities on the principal moments of inertia.

% Through this insight, semi-definite programming approaches can be used to solve \revone{inertial} parameter identification problems to global optimality. %These results suggest that constraints in rigid-body \revone{inertial} parameter identification are most directly addressed through a perspective of statistics. %Computational results are provided for system identification of a simulated Cheetah 3 robot, currently in construction at MIT.
\end{abstract}

% % Note that keywords are not normally used for peerreview papers.

\maketitle

\vspace{-16px}
\section{Introduction}
\vspace{-0px}

\IEEEPARstart{A}{dvances} in whole-body control of legged robots \cite{Abe07,Park07,SentisKhatib10,Wensing13} have led to increased use of model-based methods in experimental hardware \cite{Hutter14,Mansard15b,Kuindersma15,Herzog2016}. 
Commonly, state-of-the-art controllers perform optimization over actuator torques to generate desired motions in the robot. Recent strides in torque-controlled actuation have provided wide benefit to these techniques, and emerging actuator designs~\cite{Wensing17b} suggest that performance will continue to improve.
Despite these advances, the performance of whole-body control methods remains dependent on accurate dynamic models.
	
Thus, recent trends in control have been accompanied by many parallel developments in the area of system identification \cite{Pucci15,Ayusawa14,Jovan16,Sousa14,Traversaro16}.
This recent work follows a rich history of research into \revone{inertial} parameter identification, with seminal work in \cite{Atkeson86}.  Classically, studies have focused on challenges such as designing trajectories for optimal identification \cite{Gautier91,Swevers97}, limiting bias from structured noise \cite{Ting06,Janot14}, and verifying robustness bounds on the identified parameters \cite{Calafiore00,Poignet05}. By comparison, recent work in the legged locomotion community has focused on challenges from the floating-base structure of legged robots \cite{Pucci15,Ayusawa14} and the use of constrained optimization to limit the search space to physically realistic parameters \cite{Jovan16,Sousa14,Traversaro16}.
	
	%This problem has received a great deal of attention in the manipulator community, with results spanning the decades \cite{,Janot14}, and has close mathematical connections to the problem of adaptive control \cite{Slotine87}. 
		
% An important set of model identification problems relates to the \revone{inertial} parameters of rigid-body systems \cite{Atkeson86}. 
% Identification is often pursued offline, however, the mathematics of offline identification have close connections to online adaptive control \cite{Slotine87}. 
%while using parameter estimates in

	The floating-base structure of legged systems introduces challenges and opportunities to system identification. Pucci et al.~\cite{Pucci15} addressed coupling between the limbs and body to extend methods of Slotine and Li \cite{Slotine87} for underactuated adaptive control. Other work has exploited the fact that the Newton and Euler equations of the entire robot are embedded in the dynamics of the floating base. With this property, Ayusawa et al.~\cite{Ayusawa14, Ayusawa11} demonstrated that full-body \revone{inertial} parameters can be estimated from contact forces and kinematics alone. These advances open the door for application to systems where joint torque measurements are not available, such as for identifying inertial parameters in humans.

%due to properties of the floating-base dynamics. %They further showed that the whole-body \revone{inertial} parameters can be fully determined up to a scale factor through observations of the kinematics alone in flight \cite{Ayusawa14}.  
	
\begin{figure}
\center
\vspace{-6px}
\includegraphics[width = .38 \columnwidth]{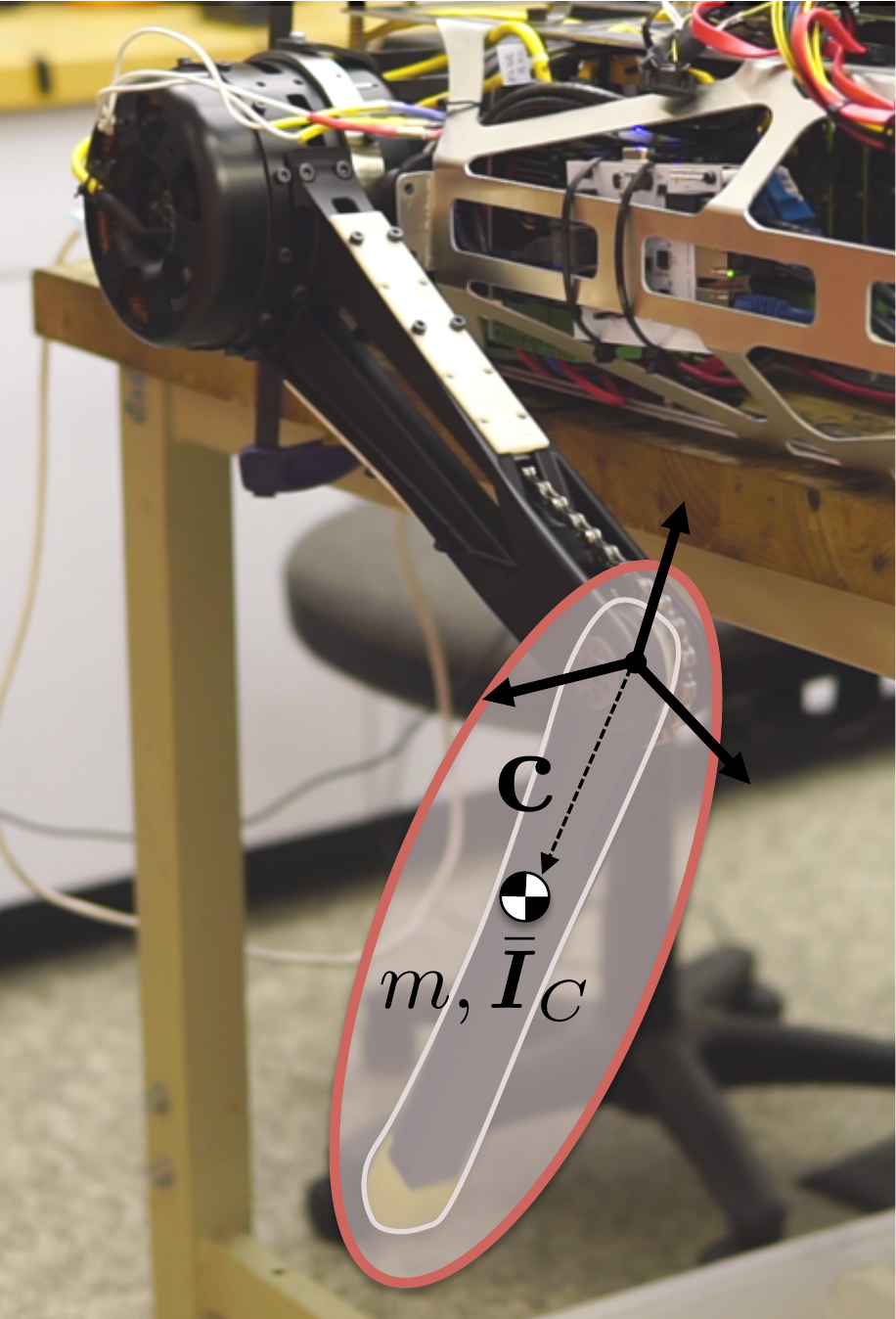}\vspace{-5px}
\caption{Setup for system identification on the MIT Cheetah 3. \q{caption1}{\revone{This work provides new constraints to identify physically plausible masses $m$, center of mass locations $\vc$, and rotational inertias $\vIbar_{C}$ through convex optimization.}} % though a mass distribution within a bounding ellipsoid. %\revone{Inertial} parameters are identified alongside transmission components through convex optimization.}
}
\label{fig:leg}
\vspace{-13.5px}
\end{figure}	
	
	Other recent work has concentrated on using constrained optimization in identification. Often, parameters are known a priori to exist within some predefined set. Such restrictions may come from considerations of physical plausibility, from bounding volumes in CAD models, or from known symmetry of a mechanism. Regardless of the source, prior knowledge can be applied to increase robustness to noisy sensors, and to accelerate model convergence \cite{Ioannou12,Slotine86}. Convex constraints, in particular, provide desirable structure that can be naturally exploited in parameter identification and adaptive control alike. For identification, Jovic et al.~\cite{Jovan16} formed bounding constraints on the center of mass (CoM) of each link, and enforced symmetry of mirrored limbs. Sousa and Cortes\~ao \cite{Sousa14} formulated LMIs to enforce positivity of the kinetic energy. This constraint alone does not guarantee physical consistency. Recently, Traversaro et~al.~\cite{Traversaro16} described tight conditions for physical consistency of \revone{inertial} parameters using a nonconvex parameterization. Optimization on manifolds was required to enforce these conditions. 

	%To enforce these conditions, methods for local optimization on nonlinear manifolds were applied.} %which limited optimality guarantees. 
	
%\begin{table}
%\center
%\begin{tabular}{cc}
%\end{tabular}
%\end{table}

\q{contribution1}{\revone{A main contribution of this paper is to show how physical-consistency constraints can be expressed as LMIs on the \revone{inertial} parameters.}}
\revone{It is shown that manifold constraints from \cite{Traversaro16} can be reformulated as convex constraints through LMIs.}
\q{ourmethod}{\revone{The proposed LMIs uniquely enable globally-optimal least-squares parameter identification while enforcing plausibility of the result.}}
\q{stochasticpre}{\revone{From the form of the constraints, it is critically observed that physical consistency depends only on the covariance of the mass distribution.}}

\q{stochastic}{\revone{This main observation stems from mathematical commonality between mass measures of rigid bodies and probability measures of random variables.
This is not to say that a rigid-body has inherent stochasticity. But rather, that certifying the plausibility of moments falls to the same mathematics in both domains.}}
\q{contribution2}{\revone{This connection enables results from the classical problem of moments \cite{Fialkow10} to yield new bounding-ellipsoid constraints on the mass distribution. Formulation of these constraints is a second main contribution of the work.}} 
\q{toInfinity}{\revone{Although constraints are placed on the mass distribution, itself an infinite dimensional mathematical object, LMIs enable this restriction to be imposed directly on the 10 standard \revone{inertial} parameters during system identification.}}

%Throughout, a statistical interpretation of the mass distribution leads to richer constraints on the \revone{inertial} parameters which are tighter than those in earlier work. %Perhaps surprisingly, these constraints also of a simpler form and are more efficient to enforce in numerical optimization.

	%Thus, the new development shows that the difference in positive definite constraints between these 6D and 4D inertia matrices is precisely the added consideration of the triangle inequalities.

	 %Identification in \cite{Sousa14} enforces positive definiteness of the 6D spatial inertia tensor \cite{Featherstone08}, while \cite{Traversaro16} argues the need to ensure additional triangle inequality constraints on the principal moments of inertia. We show that these additional constraints can be included to arrive at an LMI that is, surprisingly, of lower dimension~(4D). 	
%\revone{inertial} parameter identification problems can be solved to global optimality in full generality and correctness.  %The conditions that govern a 6D inertia tensor representing a rigid-body underlying mass distribution are also likely of general interest in mechanics. 
%}

%It is important to highlight how this work compares to previous methods that have relied on the statistics of noise. 

The paper is laid out as follows. Section \ref{sec:prelim} provides mathematical preliminaries. Section \ref{sec:theory} reviews \cite{Sousa14} and \cite{Traversaro16}, with a high-level comparison in Table \ref{tab:compare}. Section~\ref{sec:mainresults} presents LMIs for physical consistency, highlighting connections with probability and statistics. %, culminating in the (re)introduction of the pseudo-inertia matrix. %Noting that this matrix is precisely a matrix of second moments,  
Section~\ref{sec:Extensions} draws on these connections to introduce bounding-volume constraints. Section~\ref{sec:results} describes application to identify a leg from the MIT Cheetah~3, shown in Fig.~\ref{fig:leg}. Section \ref{sec:conclusions} provides concluding remarks.

%Belta - \cite{Belta02}
%Section \ref{sec:results} provides results on a simulated version of the MIT Cheetah 3 which is currently in construction at MIT, showing the promising numerical performance of the methods for system identification.

\vspace{-5px}
\section{Preliminaries}
\vspace{-0px}
\label{sec:prelim}
%This section details the mathematical preliminaries necessary to the main theoretical development of the text. Notation is first described, along with a short introduction to Linear Matrix Inequalities (LMIs).  A primer on rigid body dynamics is also provided, with focus on the role of the rigid-body \revone{inertial} parameters to characterize the equations of motion.

\subsection{Notation and Definitions}

The set of real and natural numbers are denoted by $\mathbb{R}$ and $\mathbb{N}$ respectively.   $\mathbb{R}_{+}$ represents the set of non-negative reals. Scalars are denoted with italics $(a,b,\ldots)$, vectors with bold characters $({\bf a}, {\bf b},\ldots)$, and matrices with bold capitals $({\bf A}, {\bf B}, \ldots)$. The $n\times n$ identity is noted as $\bone_n$.
%Matrix Lie groups and their associated Lie algebras are indicated with uppercase and lowercase characters $\{\mathsf{A}, \mathsf{B},...\}$ and $\{\mathsf{a}, \mathsf{b},...\}$ respectively. 
% = \{\vR \in \mathbb{R}^{3\times3} \,:\, \vR\T\vR = \bone_3,\, \det(\vR)=1 \}$, 
%=\left\{\begin{bmatrix} \vR & \vp \\ \bzero & 1 \end{bmatrix} ~:~ \vR \in \SO{3},\, \vp\in\mathbb{R}^3\right\}                              
The Special Orthogonal group of rotations is denoted $\SO{3}$, with its Lie algebra, the set of $3\times3$ skew-symmetric matrices $\so{3}$. The Special Euclidean group  is denoted $\SE{3}$ with its Lie algebra  $\se{3}$.
%\[
%\se{3} = \left\{ \begin{bmatrix} \boldsymbol{\Omega} & \vel \\ \bzero_{1\times3} & 0\end{bmatrix} ~:~\boldsymbol{\Omega} \in \so{3},\, \vel\in\mathbb{R}^3 \right \}
%\]
%The set of invertible $n\times n$ matrices is denoted by $\GL{n}$ with its Lie algebra $\mathbb{R}^{n\times n}$. 
The set of symmetric $n\times n$ matrices is represented as $\mathbb{S}^n$, the positive semidefinite cone $\mathbb{S}^n_+$, and the positive definite cone $\mathbb{S}^n_{++}$ \cite{BoydVanberghe04}. The shorthand $\vA \succeq \vB$ indicates $\vA-\vB \in \mathbb{S}_+^n$ for some $n\in\mathbb{N}$.  $\vA \succ \vB$ similarly indicates $\vA-\vB \in \mathbb{S}^n_{++}$.
%For a square matrix, $\bA \in \mathbb{R}^{n\times n}$, the characteristic polynomial of $\bA$ is defined by $p_\bA(\lambda)=\det(\lambda \bone_{n\times n} - \bA)$ where $\bone_{n \times n}$ is the identity matrix in $\mathbb{R}^{n\times n}$. The roots of this polynomial (the eigenvalues of $\bA$) are denoted by the set $\sigma_{\bf A} := \{\lambda \in \mathbb{C} ~:~ p_{\bf A}(\lambda) = 0\}$\,, with the spectral radius specified as $\rho(\bA) := \max(\{|\lambda| ~:~ \lambda \in \sigma_\bA\})$. The tangent bundle of a manifold $\mathcal{Q}$ is denoted as $T\mathcal{Q}$. Table \ref{tab:Notation} details common symbols, subscripts, and superscripts used throughout the text.
\begin{definition}[LMI Representable]
A convex set $\mathcal{S} \subset \mathbb{R}^n$ is called linear matrix inequality (LMI) representable if there exists $m\in \mathbb{N}$ and constant matrices $\{\vA_i\}_{i=0}^n \in  \mathbb{S}^m$ with
\[
\mathcal{S} = \{ \vx \in \mathbb{R}^n ~:~ \vA_0 + x_1 \vA_1 + \cdots + x_n \vA_n \succeq \bzero\}
\]
A set is called {\em strictly} LMI representable when the inequality can be tightened to hold strictly.
\end{definition}
%\noindent A set that is LMI representable is always convex, however the reverse is not always true. 
\noindent \q{lmiinfo}{\revone{As a main benefit, a convex set $\mathcal{S}$ being LMI representable has favorable implications for optimization. Constraints of the form $\vx \in \mathcal{S}$ can be enforced using semidefinite programming~\cite{BoydVanberghe04}. Such techniques are mature, and admit guarantees of global optimality. We refer the interested reader to \cite{Boyd94} for further background on LMIs and their applications.} }  
%\cite{Helton07}
%% on LMIs and their applications.

\begin{table}
\center
\vspace{6px}
\begin{tabular}{ccccc}
\hline 
                       & Physical        &          & Discrete & Uses \\
                       & Consistency    &  Convex  & Approx.  & LMIs \\ \hline
Sousa {\em et al.} \cite{Sousa14}   &          & \checkmark  & No   & \checkmark            \\
Traversaro {\em et al.} \cite{Traversaro16} & \checkmark & & No &  \\
Ayusawa {\em et al.}  \cite{Ayusawa11}   & \checkmark & \checkmark & Yes &\\ 
This Paper			    & \checkmark & \checkmark & No & \checkmark\\ \hline\\[-1ex]
\end{tabular}
%\caption{Comparison of features amongst parameter identification optimization problems formulated in the literature. }
\caption{Feature comparison. }
\label{tab:compare}
\vspace{-20px}
\end{table}

\revone{
\begin{definition}[Moment of a Measure] Consider a positive Borel measure $\mu$ on $\mathbb{R}^n$, $\boldsymbol{\alpha} = [\alpha_1, \ldots, \alpha_n]\T \in \mathbb{N}_+^n$, and let $\vx^{\boldsymbol{\alpha}} = x_1^{\alpha_1} \cdots x_n^{\alpha_n}$ for any $\mathbf{x} = [x_1, \ldots, x_n]\T \in \mathbb{R}^n$. Then
%$\int \mathbf{x}^{\boldsymbol{\alpha}}\, {\rm d} \mu(\vx)$
\[
% \gamma(\boldsymbol{\alpha}) = 
 \mbox{\Large$\int$}_{\!\!\mathbb{R}^n}\, \mathbf{x}^{\boldsymbol{\alpha}}\, {\rm d} \mu(\vx)
\]
is called a moment of $\mu$, of order $|\boldsymbol{\alpha}| = \sum\limits_{i=1}^{n} \alpha_i$~~~\cite{Fialkow10}.
\label{def:moment}
\end{definition}
\noindent Any function $f:\mathbb{R}^n \rightarrow \mathbb{R}_+$ can be used to define a measure through the association ${\rm d} \mu(\vx) = f(\vx) {\rm d} \vx$, where ${\rm d} \vx$  represents a differential volume.}
\q{conceptMoment}{\revone{The general concept of a moment is applicable to both probability measures and mass measures. This commonality will be used for intuition into the LMIs that enforce plausibility of \revone{inertial} parameters. }}
 
%Given a sequence of moments $\sigma_\boldsymbol{\alpha}$ 
\vspace{-10px}
\subsection{Rigid-Body Dynamics}
%\vspace{-2px}

%The treatment of rigid-body dynamics here relies heavily on 6D spatial notation \cite{Featherstone08,FeatherstoneOrin08}. Spatial notation can be one-to-one translated to a corresponding Lie-theoretic notation~\cite{Park95}. A recent technical report nicely describes notational and conceptual correspondences~\cite{Traversaro16b}.
%The reference \cite{Featherstone08} provides an excellent and detailed account of spatial notation, with a more practitioner-focused treatment in \cite{FeatherstoneOrin08}.

%The dynamics of $n_b\in \mathbb{N}_+$ rigid bodies follow \cite{Featherstone08}
The dynamics of a system of $n_b\in \mathbb{N}_+$ rigid bodies follows
\begin{equation}
\vH(\vq)\, \vnud + \vC(\vq,\vnu)\, \vnu + \vg(\vq) = \btau
\label{eq:eom}
\end{equation}
where $\vH \in \mathbb{R}^{n_d \times n_d}$ the mass matrix, $n_d \in \mathbb{N}_+$ the number of degrees of freedom, $\vq \in \mathcal{Q}$ the configuration with $\mathcal{Q}$ the configuration manifold, $\vnu \in \mathbb{R}^{n_d}$ the generalized velocity, $\vC\,\vnu \in \mathbb{R}^{n_d}$ and $\vg \in \mathbb{R}^{n_d}$ the Coriolis and gravity forces, and $\boldsymbol{\tau} \in \mathbb{R}^{n_d}$ the generalized force \cite{Featherstone08}. For legged systems, the generalized force $\btau$ has contributions from $n_j\in\mathbb{N}_+$ joint actuator torques $\btau_j\in\mathbb{R}^{n_j}$ and $n_c \in \mathbb{N}$ external contact wrenches $\{ \vf_{c_k} \}_{k=1}^{n_c} \subset \mathbb{R}^6$
according to
\[
\btau = \vS_j\T \btau_j +  \mbox{$\sum\limits_{k=1}^{n_c}$} \vJ_{c_k}\T \vf_{c_k}
\]
where $\vS_j \in \mathbb{R}^{n_j \times n_d}$ is an actuated joint selector matrix and $\vJ_{c_k} \in \mathbb{R}^{6 \times n_d}$ the 6D Jacobian for contact $k$.  

\revone{It is commonly known that \eqref{eq:eom} can be represented linearly in system \revone{inertial} parameters $\vpi \in \mathbb{R}^{10 n_b}$ \cite{Atkeson86}
\begin{equation}
\vH(\vq)\, \vnud + \vC(\vq,\vnu)\, \vnu + \vg(\vq) = \vY(\vq, \vnu, \vnud)\, \vpi 
\label{eq:regressor}
\end{equation}
where $\vY$ is the regressor matrix. The parameters $\vpi$ have contributions from each body such that $\vpi = [\vpi_i\T, \ldots, \vpi_{n_b}\T]\T$. The body \revone{inertial} parameters $\vpi_i \in \mathbb{R}^{10}$ are composed as
\[
\vpi_i = [m, h_x, h_y, h_z, I_{xx}, I_{xy}, I_{xz}, I_{yy}, I_{yz}, I_{zz}]\T \in \mathbb{R}^{10}
\]
with $m$ the body mass, $\vh = [h_x, h_y, h_z]\T = m \vc$ the first mass moment with $\vc \in \mathbb{R}^3$ the vector to the CoM in a body-fixed coordinate system, and   
\[
\vIbar = \left[ \begin{smallmatrix} I_{xx} & I_{xy} & I_{xz} \\ I_{xy} & I_{yy} & I_{yz} \\ I_{xz} & I_{yz} & I_{zz} \end{smallmatrix} \right]
\]
the rotational inertia about the coordinate origin. These parameters also describe the 6D (spatial) body inertia \cite{Featherstone08}
\begin{equation}
\vI_i = \begin{bmatrix} \vIbar_i & m_i\vS(\vc_i) \\ m_i \vS(\vc_i)\T & m_i \bone_3 \end{bmatrix} = \begin{bmatrix} \vIbar_i & \vS(\vh_i) \\ \vS(\vh_i)\T & m_i \bone_3 \end{bmatrix}
\label{eq:spatial6}
\end{equation}
with $\vS(\vx) \in \so{3}$ such that $\vS(\vx) \vy = \vx \times \vy$,  $\forall\,\vx,\vy\in\mathbb{R}^3$. }

 The regressor matrix provides a simple method to pursue \revone{inertial} parameter identification. 
Given $n_s\in\mathbb{N}_+$ samples, a  least-squares identification problem can be formulated \cite{Sousa14}
\begin{equation}
 \min_{\vpi}~ \mbox{$\sum\limits_{m=1}^{n_s}$}~ \| \vY^{(m)}\, \vpi -  \boldsymbol{\tau}^{(m)}  \|^2
\label{eq:opt}
\end{equation}
This optimization problem is efficiently solvable to global optimality. However, without including constraints, the optimal parameters may not correspond to any physical system.

\revone{Inertial} parameters in any physical body are determined by a distribution of density $\rho_i(\cdot) : \mathbb{R}^3 \rightarrow \mathbb{R}_{+}$. The \revone{inertial} components for each body $i$ are a functional of $\rho_i(\cdot)$ \cite{Traversaro16}
%\begingroup\makeatletter\def\f@size{9}\check@mathfonts
\begin{align}
m_i &= \int_{\mathbb{R}^3} \rho_i(\vx)\, \textrm{d}\vx \label{eq:density1}\\
\revone{\vh_i}  &= \int_{\mathbb{R}^3} \vx\, \rho_i(\vx)\, \textrm{d}\vx   \label{eq:density2}\\
\revone{\vIbar_i} %&= \int_{\mathbb{R}^3} \vS(\vx)\, \vS(\vx)\!\T \rho_i(\vx) \,\textrm{d}\vx  
&\revone{=  \int_{\mathbb{R}^3} \underbrace{ \left[ \begin{smallmatrix}y^2 + z^2 & -xy & -xz \\ -xy & x^2 + z^2 & -yz \\ -xz & -yz & x^2+y^2 \end{smallmatrix} \right]}_{\vS(\vx)\vS(\vx)\T} \rho_i(\vx)\,\textrm{d}\vx }
\label{eq:density3}
\end{align}
%\begingroup\makeatletter\def\f@size{10}\check@mathfonts
\revone{The moments of inertia within $\vIbar_i$ are not moments of $\rho(\cdot)$ in the sense of Definition~\ref{def:moment}. While $\vIbar_i$ is convenient to describe dynamics, the physical plausibility of \revone{inertial} parameters will be more directly addressed with moments as in Definition~\ref{def:moment}.}
%Note that $m_i$ and the entries of $\vh_i$ are 0-th order and 1st order moments of $\rho(\cdot)$, whereas the entries of $\vIbar_i$ are not.
\begin{definition}[Density Realizable] Given a set $\mathcal{X} \subseteq \mathbb{R}^3$, a 6D inertia  $\vI$ is called $\mathcal{X}$-density realizable if $\exists \rho(\cdot):\mathbb{R}^3 \rightarrow \mathbb{R}_{+}$ such that $\rho(\vx)=0$ when $\vx \notin \mathcal{X}$, and the components of $\vI$, $(m, \revone{\vh}, \vIbar)$, satisfy \eqref{eq:density1}-\eqref{eq:density3}. When $\mathcal{X}$ is not specified, $\mathcal{X}=\mathbb{R}^3$ is assumed.
\end{definition}

\begin{remark} Given any $\mathcal{X}\subseteq \mathbb{R}^3$, the set $\mathcal{P}^*_{\mathcal{X}}$ defined by
$\mathcal{P}^*_\mathcal{X} = \{\vpi \in \mathbb{R}^{10} \,:\, m(\vpi)>0,~ \vI(\vpi) {\rm~is~} \mathcal{X}{\rm-density~realizable}\}$ is a convex cone.  
%This follows from the fact that \eqref{eq:density1}-\eqref{eq:density3} are linear in $\rho(\cdot)$. 
Previous work \cite{Ayusawa11} provided a discrete approximation to this cone. Without discretization, the work here provides cases wherein the cone is LMI representable.  
%This can be reasoned physically. Let $\vpi_1$, $\vpi_2 \in \mathcal{P}^*_\mathcal{X}$ with density distributions $\rho_1(\cdot)$ and $\rho_2(\cdot)$. 
%For $\alpha>0$, $\alpha \vpi_1$ can be represented by scaling the underlying density distribution $\alpha \rho_1(\cdot)$. 
%For a convex combination  $\alpha \vpi_1 + (1-\alpha) \vpi_2$  with $0\le \alpha \le 1$, scaling and adding density distributions as $\alpha \rho_1(\cdot) + (1-\alpha) \rho_2(\cdot)$ verifies $\alpha \vpi_1 + (1-\alpha) \vpi_2 \in \mathcal{P}^*_\mathcal{X}$.
\end{remark}

%Answering a question of density realizability for a given set of parameters $\vpi_i$ amounts to verifying the existence of an infinite dimensional object, the underlying density distribution $\rho(\cdot)$. The following section describes necessary constraints on each $\vpi$ to ensure density representibility in the case when $\mathcal{X} = \mathbb{R}^3$.

\vspace{-3px}
\section{Previous Results}
\label{sec:theory}
\vspace{-0px}

This section focuses on physical consistency for a single rigid body. As such, body indices will be dropped. Attempts to enforce physical consistency focus on the rotational inertia $\vIbar_C$ about the CoM. The parallel axis theorem in 3D establishes a correspondence between $\vIbar_C$ and $\vIbar$, the rotational inertia about a body-fixed coordinate origin, through
\begin{equation}
\vIbar = \vIbar_C + m \vS(\vc) \vS(\vc)\T
\label{eq:parallelaxis}
\end{equation}

\subsection{Physical Semi-consistency: An LMI Parameterization}
Positive definite constraints on $\vIbar_C(\vpi)$ have been commonly enforced   on the \revone{inertial} parameters $\vpi$ \cite{Sousa14,Slotine89b}. In a rigid-body system, when each $m_i>0$ and $\vIbar_{C_i} \succ 0$, it can be shown that $\vH(\vq) \succ 0$ $\forall\vq \in \mathcal{Q}$ \cite{Slotine89b}. %In the case of adaptive control, satisfaction of this constraint on $\vH(\vq)$ throughout adaptation is sufficient to admit tracking guarantees \cite{Slotine89b}. 
\revone{However, these constraints are not alone enough to ensure physical consistency of the \revone{inertial} parameters \cite{Traversaro16}.}

%$\vIbar_{C}(\vpi) \succ 0$ is not sufficient for density realizability of $\vI(\vpi)$. %With this in mind, the following definition is used here. 

\begin{definition}[Physical Semi-consistency] A vector of \revone{inertial} parameters $\vpi \in \mathbb{R}^{10}$ is called {\em  physically semi-consistent} if $m(\vpi) > 0$ and $\vIbar_C(\vpi) \succ 0$. The set of physically semi-consistent parameters is denoted ${\mathcal{P} \subset \mathbb{R}^{10}}$.
\end{definition}

\begin{theorem}[LMI Representation of $\mathcal{P}$] \label{lmi:semi}\cite{Sousa14}
The set of physically semi-consistent parameters $\mathcal{P}$ is  strictly LMI representable. Its LMI representation is given as
$\mathcal{P} = \left \{ \vpi \in \mathbb{R}^{10} \,:\, \vI(\vpi) \succ 0\right\}$. 
\end{theorem}
Extending the optimization \eqref{eq:opt} to include physical semi-consistency constraints results in a semidefinite programming (SDP) problem. This problem can be solved to global optimality with SDP solvers \cite{BoydVanberghe04}
%Recall that the inertia tensor is linear in the parameters $\vpi$
%\begin{equation}
%\vI(\vpi) = \begin{bmatrix} \vIbar & \vS(\vh) \\ \vS(\vh)\T & m \bone_3 \end{bmatrix}
%\end{equation}
%The Schur complement of $m \bone_3$ in $\vI$ \cite[Section A.5.5]{BoydVanberghe04} is given by
%\begin{equation}
%\vIbar - \tfrac{1}{m}\vS(\vh) \vS(\vh)\T
%\end{equation}
%Yet, since $\vS(\vh) = m \vS(\vc)$ it follows that
%\begin{equation}
%\vIbar - \tfrac{1}{m}\vS(\vh) \vS(\vh)\T = \vIbar - m \vS(\vc) \vS(\vc)\T = \vIbar_C
%\end{equation}
%Thus, via the Schur Complement lemma \cite{BoydVanberghe04}, $\vpi \in \mathcal{P}$ if and only if $\vI(\vpi) \succ 0$.
%\end{proof}
\begin{align}
\min_{\vpi} ~& \sum_{m} \| \vY^{(m)}\, \vpi -  \boldsymbol{\tau}^{(m)} \|^2 \nonumber \\ 
			\textrm{s.t.}~& \vI(\vpi_i) \succ 0 \quad \forall i \in \{1,\ldots,n_b\} \nonumber
\end{align}

\subsection{(Full) Physical Consistency: Manifold Parameterization}
\begin{definition}[Physical Consistency] A vector of \revone{inertial} parameters $\vpi \in \mathbb{R}^{10}$ is called {\em physically consistent} if $m(\vpi) > 0$  and  $\vI(\vpi)$ is density realizable.
The set of physically-consistent parameters is denoted $\mathcal{P}^*$.
\end{definition}

In comparison to $\mathcal{P}$, the set of physically consistent \revone{inertial} parameters $\mathcal{P}^* \subset \mathcal{P}$ has been shown  to result from only three additional conditions on $\vIbar_C$ \cite{Traversaro16}. These additional constraints arise from considerations regarding the principal moments of inertia. Suppose $\vR \in \mathsf{SO}(3)$ and $\vJ = \textrm{diag}(J_1, J_2, J_3)$, $J_1,...,J_3>0$ such that $\vIbar_C = \vR \vJ \vR\T$. Then, the rotational inertia $\vIbar_C$ is density realizable iff\begin{equation}
J_1+J_2 \ge J_3,~ J_2+J_3 \ge J_1,~\textrm{and } J_1 + J_3 \ge J_2 
\label{eq:triangleIneq}
\end{equation}
\begin{definition}[Triangle Inequalities]
A matrix in $\mathbb{S}^3$ is said to satisfy the triangle inequalities if its eigenvalues $\{J_i\}_{i=1}^3$ satisfy \eqref{eq:triangleIneq}.
\end{definition}

\begin{theorem}[Manifold Parameterization of $\mathcal{P}^*$] \cite{Traversaro16} 
A 6D inertia $\vI$ is physically consistent if and only if there exists $m>0$, $\vR \in \mathsf{SO}(3)$, $\vJ = {\rm{diag}}(J_1, J_2, J_3) \succ 0$ that satisfies the triangle inequalities, and $\vc \in \mathbb{R}^3$ such that 
\[
\vI = \begin{bmatrix} \vR \vJ \vR\T + m \vS(\vc) \vS(\vc)\T & m \vS(\vc) \\ m \vS(\vc)\T & m \bone_3 \end{bmatrix}
\]
\end{theorem}

Even for a single rigid body, optimization with this parameterization of $\mathcal{P}^*$ results in a nonlinear optimization problem over a manifold
\begin{align}
\min_{\vR, \vJ, \vc, m} ~&\sum_{m} \| \vY^{(m)}\, \vpi(\vR, \vJ, \vc, m) -  \boldsymbol{\tau}^{(m)} \|^2 \nonumber\\
\textrm{s.t.}~& \vR \in \mathsf{SO}(3)  \nonumber\\
			  & m>0 ,~ J_i>0, i=1,2,3 \nonumber \\
			  & J_1+J_2 \ge J_3,~ J_2+J_3 \ge J_1,~\textrm{and } J_1 + J_3 \ge J_2 \nonumber
\end{align}
\q{Traversaro}{Solution of this problem is possible using nonlinear optimization on manifolds~\cite{Traversaro16}. However, this approach requires custom solvers and does not guarantee global optimality.}

\section{Contribution: An LMI for Physically Consistent \revone{Inertial} Parameters}
\label{sec:mainresults}
\revone{This section takes a closer look at conditions for physical consistency. It is shown that the triangle inequalities can be expressed as an LMI over $\vpi$, without a manifold parametrization. First, Section \ref{subsec:LMIIC} describes a matrix inequality for the triangle inequalities on $\vIbar_C$.} Intuition into this result is given in Section \ref{subsec:dwcov} through introduction of the {\em density-weighted covariance} of a rigid body and its \revone{\em covariance ellipsoid}. With this interpretation, Section \ref{sec:finallmi} develops an LMI over $\vpi$ for physical consistency. \revone{As a key benefit, this LMI enables the use of convex optimization to identify plausible parameters.} 
%This physical interpretation will be shown to admit a LMI characterization of density realizability through an analog of the parallel axis theorem as applied to a matrix of second moments.

\subsection{A matrix inequality for triangle inequalities on $\vIbar_C$}
\label{subsec:LMIIC}

Suppose $\vR$ and $\vJ$ as before such that $\vIbar_C = \vR \vJ \vR\T$. The triangle inequalities on $\vIbar_C$ \eqref{eq:triangleIneq} can be rewritten as
\begin{equation}
 J_1 +J_2+J_3 \ge  2J_i~~~~\textrm{$i=1,...,3$}
\label{eq:newtriangle}
\end{equation}
\revone{Since $J_i$ are the eigenvalues of $\vIbar_C$, \eqref{eq:newtriangle} is equivalent to
\begin{equation}
\tfrac{1}{2}\textrm{Tr}( \vIbar_C ) \ge \lambda_{max}(\vIbar_C) 
\label{eq:maxeig}
\end{equation}
where $\textrm{Tr}(\cdot)$ is the trace operator, and $\lambda_{max}(\cdot)$ provides the maximum eigenvalue of its argument.} Separately, the eigenvalue inequality
$\lambda_{max}(\vIbar_C) ~\mathbf{x}\T \mathbf{x} \ge  \mathbf{x}\T \, \vIbar_C \,  \mathbf{x} $ implies
\begin{equation}
\lambda_{max}(\vIbar_C) \bone_3\succeq \vIbar_C 
\label{eq:maxEig2}
\end{equation}
Thus, through the use of \eqref{eq:maxEig2}, \eqref{eq:maxeig} is equivalent to
\begin{equation}
\tfrac{1}{2} \textrm{Tr}( \vIbar_C )\, \bone_3 - \vIbar_C \succeq 0 
\label{eq:lmiatc}
\end{equation}
\revone{Although \eqref{eq:lmiatc} is mathematically equivalent to the triangle inequalities on $\vIbar_C$, its intuitive meaning is hardly clear in this form. The next section builds towards this intuition.}

\vspace{-4px}
\subsection{The Density-Weighted Covariance of a Rigid Body}
\vspace{-2px}
\label{subsec:dwcov}
\revone{It will be shown that the mathematical condition \eqref{eq:lmiatc} can be interpreted as requiring a positive semidefinite covariance of the rigid-body mass distribution.} Towards this insight, $\vIbar_C$ can be expanded algebraically to verify 
\begin{align}
\vIbar_C &= \mbox{\Large$\int$}_{\!\!\mathbb{R}^3}\, \vS(\vx_c) \vS(\vx_c)\!\T \rho(\vx) \,\textrm{d}\vx 
\nonumber\\
 &=\mbox{\Large$\int$}_{\!\!\mathbb{R}^3} \left(  \rm{Tr}(\vx_c\vx_c\T) \bone_3 - \vx_c\vx_c\T \right) \rho(\vx) \,\textrm{d}\vx 
 % \underbrace{}_{ \vS(\vx_c) \vS(\vx_c)\!\T}
\label{eq:IbarCexpand}
\end{align}
where $\vx_c = \vx-\vc$. To simplify this expression, the {\em density-weighted covariance} of a rigid body is introduced as
\begin{equation}
\boldsymbol{\Sigma}_C = \mbox{\Large$\int$}_{\!\!\mathbb{R}^3} \, \vx_c\,\vx_c\T \,\rho(\vx)\, \textrm{d}\vx
\label{eq:SigmaC}
%\int_{\mathbb{R}^3}
\end{equation}
\revone{Note that $\vc$ is the mean position of the rigid body, in a density-weighted sense. Thus, when $m=1$, the definition \eqref{eq:SigmaC} matches that of covariance in probability and statistics.}

\revone{The covariance $\boldsymbol{\Sigma}_C$ and rotational inertia $\vIbar_C$ are related through a rich set of properties. Algebraically, from \eqref{eq:IbarCexpand},
\begin{align}
\vIbar_C &= \rm{Tr}( \boldsymbol{\Sigma}_C) \bone_3 - \boldsymbol{\Sigma}_C \label{eq:IbarCasE}
\end{align}
Taking the trace of both sides provides that $\rm{Tr}(\vIbar_C) = 2\, \rm{Tr}(\boldsymbol{\Sigma}_C)$. From \eqref{eq:IbarCasE}, this property can be used to show 
\begin{align}
\boldsymbol{\Sigma}_C &= \tfrac{1}{2} \rm{Tr}(\vIbar_C) \bone_3 - \vIbar_C
\label{eq:SigmaFromIbar}
\end{align}
which matches the form of the matrix inequality \eqref{eq:lmiatc}.} \pagebreak

\q{commoneigs}{
\revone{Moreover, from \eqref{eq:IbarCasE}, it can be seen that $\boldsymbol{\Sigma}_C$ and $\vIbar_C$ share a set of eigenvectors. That is, the principal axes of $\vIbar_C$ are also the eigenvectors of $\boldsymbol{\Sigma}_C$. It can further be verified that if $\mu_1, \mu_2, \mu_3$ are the eigenvalues of $\boldsymbol{\Sigma}_C$, then 
\begin{align}
J_1 = \mu_2+\mu_3,~J_2 = \mu_1+\mu_3,~J_3=\mu_1+\mu_2
\label{eq:eigrelationships}
\end{align}
are the eigenvalues of $\vIbar_C$.}}
\q{eigenvals}{\revone{Intuitively, $\mu_1/m$ gives the average squared distance to the CoM {\em along} the direction of the first principal axis. In comparison, $J_1/m$ gives the average squared distance to the CoM {\em orthogonal to} the first principal axis.}} \revone{The eigenvalue relationships \eqref{eq:eigrelationships} show that there is a double counting of sorts when it comes to tallying the rotational moments of inertia. It is this double counting that is the source of the triangle inequalities.} 

\q{stdellipfig}
{
\begin{figure}
\center
\includegraphics[width=.87 \columnwidth]{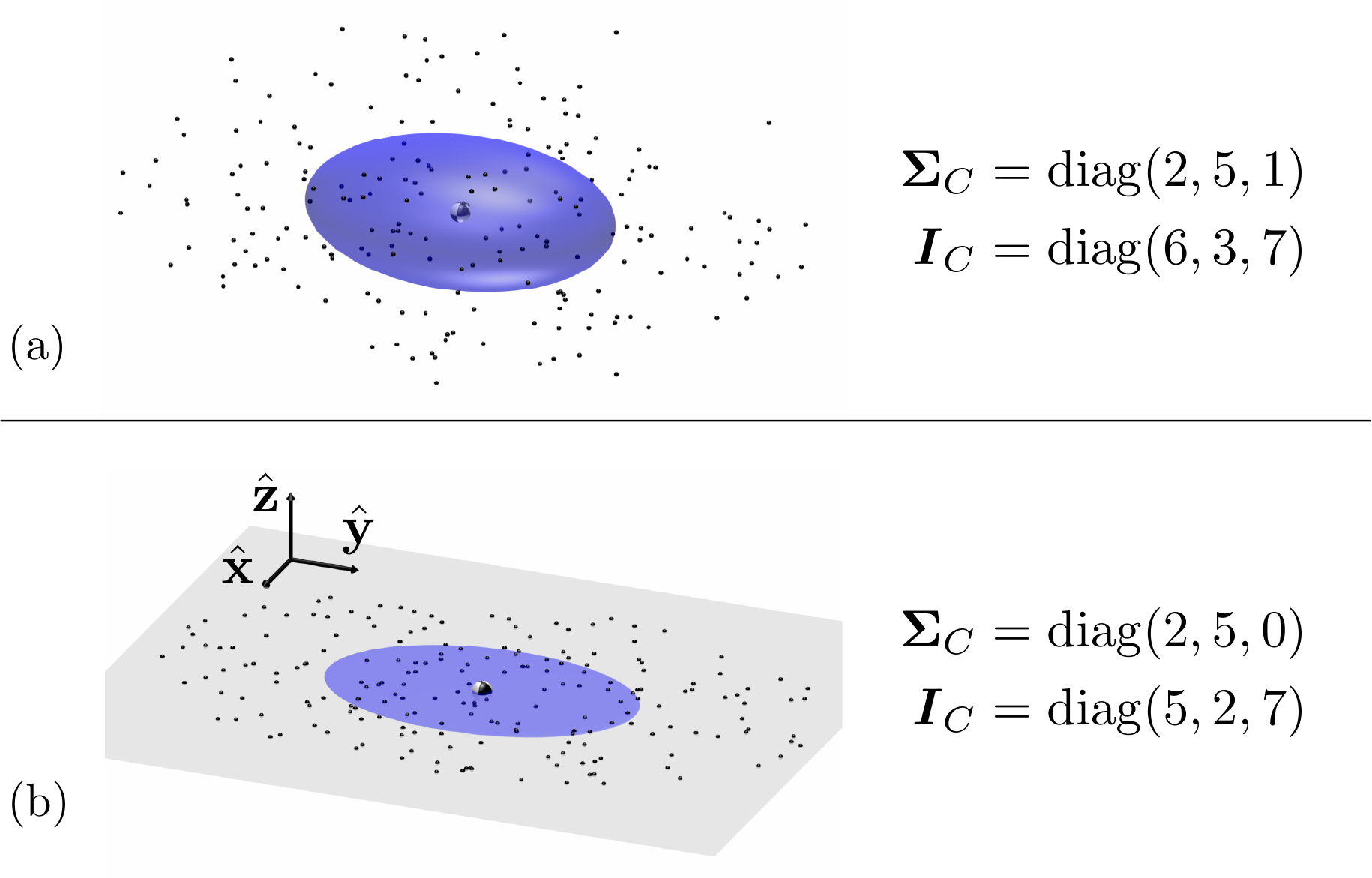}
\caption{
\revone{Graphical representation of $\boldsymbol{\Sigma}_C$. Point-mass distribution examples with (a) $\boldsymbol{\Sigma}_C\succ0$ for a distribution in 3D (all triangle inequalities hold strictly) (b) $\boldsymbol{\Sigma}_C\succeq 0$ for a distribution on an infinitely thin plate. Since distribution (b) is degenerate, $\boldsymbol{\Sigma}_C$ has one zero eigenvalue, and thus one triangle inequality is tight. The blue ellipsoid shown is the covariance ellipsoid $\mathcal{E}_{\vpi}$, and captures the shape of the distribution to second order.}}
\label{fig:GraphicalDescription}
\vspace{-5px}
\end{figure}
}

\revone{To help visualize $\boldsymbol{\Sigma}_C$, when $\boldsymbol{\Sigma}_C \succ 0$, we define}
\[
\mathcal{E}_{\vpi} = \{ \vx \in \mathbb{R}^3 \,:\,  (\vx -\vc)\T   \left( \textrm{\small$\nicefrac{ \boldsymbol{\Sigma}_C}{m}$} \right)^{-1} (\vx -\vc) \le 1 \}
\]
More generally, when $\boldsymbol{\Sigma}_C \succeq 0$ we let
\[
\mathcal{E}_{\vpi} =  \{ \vx \in \mathbb{R}^3 : \exists \vy,~  (\textrm{\small$\nicefrac{\boldsymbol{\Sigma}_C}{m}$})  \vy =   (\vx-\vc) ,~ (\vx -\vc)\T \vy \le 1 \}
\]
The set $\mathcal{E}_{\vpi}$ is named the {\em covariance ellipsoid}. Mass normalization $\nicefrac{\boldsymbol{\Sigma}_C}{m}$ in this definition allows $\mathcal{E}_{\vpi}$ to capture the shape of the distribution to second order while being invariant to uniform scaling in mass.
\q{stddevellipse}{\revone{The ellipsoid has semi-axes whose directions match the principal axes of $\vIbar_C$. However, the lengths of the semi-axes are $\sqrt{\mu_i/m}$, the root-mean-square distance to the CoM along each principal axis.}}

% This shape roughly generalizes the notion of standard deviations to 3D. In the case of a 1D probability distribution, some of the probability mass must exist at or beyond the standard deviation.  Similarly, at least some of the mass of the body must exist on or beyond $\mathcal{E}_{\vpi}$. 

%\footnote{If an academic exam has a standard deviation of 10\%, someone must have scored at least 10\% above or below the mean.}
\revone{The covariance ellipsoid $\mathcal{E}_{\vpi}$ is shown in Figure~\ref{fig:GraphicalDescription} for two example point-mass distributions. For simplicity of presentation, both distributions have principal axes that are axially aligned. When $\boldsymbol{\Sigma}_C$ has a zero eigenvalue, as in Fig.~\ref{fig:GraphicalDescription}(b), the mass distribution is degenerate. For this infinitely thin plate, the mean-squared distance to the CoM along the $\hat{\mathbf{z}}$ direction is zero. Thus, $\boldsymbol{\Sigma}_C$ has a zero eigenvalue corresponding to the $\hat{\mathbf{z}}$ principal axis. The corresponding eigenvalue in $\vIbar_C$ is non-zero, and measures the 2D rotational inertia of the plate about $\hat{\mathbf z}$. Due to the degeneracy, however, one of the triangle inequalities holds with equality.} More precisely, the eigenvalue relationships \eqref{eq:eigrelationships} verify the following. \pagebreak

\begin{proposition}[Covariance Interpretation of Triangle Ineqs.] \label{prop:cov} Suppose $\vIbar$, $\boldsymbol{\Sigma} \in \mathbb{S}^3$,  $\boldsymbol{\Sigma}= \tfrac{1}{2}\rm{Tr}(\vIbar ) \bone_3 - \boldsymbol{\vIbar}$. Then $\boldsymbol{\Sigma}\succeq0$ if and only if $\vIbar\succeq0$ and $\vIbar$ satisfies the triangle inequalities. An analogous statement holds with all inequalities strict. 
\end{proposition}

%The following corollary is an immediate result of combining Proposition \ref{prop:cov} with \eqref{eq:lmiatc}

\begin{corollary}[Parameterization of $\mathcal{P}^*$ with an LMI on $\boldsymbol{\Sigma}_C$]  $\vpi \in \mathcal{P}^*$ if and only if $m(\vpi)>0$ and $\boldsymbol{\Sigma}_C(\vpi) \succeq 0$.
\label{cor:LMISigma}
\end{corollary} 

\q{result1}{\revone{Stating Corollary \ref{cor:LMISigma} plainly, physical consistency is equivalent to the mass being positive and the density-weighted covariance being positive semidefinite.}} 
\revone{Note that since any physical rigid body is non-degenerate, physical consistency can alternately be considered requiring the density-weighted covariance to be positive definite.} 
\revone{We close this subsection with a few remarks for context.}% and motivate their broader application.

\begin{remark} A result similar to Proposition \ref{prop:cov} can be found within \cite{Belta02}, however, the connection to the density-weighted covariance provided here is new.
\end{remark}

\begin{remark}
Corollary \ref{cor:LMISigma} could have applicability to inertia identification and adaptive methods for attitude control of aerial vehicles (e.g.~\cite{Bernstein06,Thakur14}). The triangle inequalities on $\vIbar_C$ are only treated in a small subset of the literature on this topic (e.g.~\cite{ManchesterZ17}). Corollary \ref{cor:LMISigma} could be used to address the triangle inequalities in this domain.
\end{remark}

\begin{remark} 
\q{newCorrolary}{\revone{Corollary \ref{cor:LMISigma} can also be understood within the context of probability and statistics. Suppose a density function $\rho(\cdot)$ and mass $m>0$ with $\int \rho = m$. We can identify the density with a random variable $\vX \in \mathbb{R}^3$ through probability density $p(\cdot) = \rho(\cdot)/m$. This association is well posed, since $\rho(\cdot)$ non-negative implies $p(\cdot)$ non-negative, and $\int p = 1$.
Let $E[\cdot]$ the expectation operation and $\bf \Sigma(\vX) = {E[ (\vX - E[\vX]) (\vX - E[\vX])\T]}$.
Up to scaling by mass, Corollary~\ref{cor:LMISigma} is equivalent to the following. Suppose ${\mathbf E} \in \mathbb{S}^3$, then there exists a random variable $\vX \in \mathbb{R}^3$ such that ${\mathbf E} = \boldsymbol{\Sigma}(\vX)$ iff $\vE\succeq 0$. Noting that $\boldsymbol{\Sigma}(\vX)$ is the covariance of $\vX$, our conditions on density realizability for rigid bodies may be unsurprising in hindsight.}}% in hindsight.
\end{remark}

%cite Jun10

\subsection{An LMI Representation of Physical Consistency}
\label{sec:finallmi}
While \eqref{eq:lmiatc} and its covariance interpretation provide a matrix inequality for the physical consistency of $\vIbar_C$, this condition is not linear in the \revone{inertial} parameters $\vpi$. Towards an LMI over $\vpi$, %An analog of the parallel axis theorem for $\boldsymbol{\Sigma}_C$, however, will admit an LMI characterization similar to that provided in Theorem \ref{lmi:semi}.                                                                                                                                                                                              
a matrix of second moments is defined as
\begin{equation}
\boldsymbol{\Sigma} = \int_{\mathbb{R}^3} \vx\vx\T \rho(\vx) \, \textrm{d}\vx \revone{= \int_{\mathbb{R}^3} \left[ \begin{smallmatrix} \,\,x^2 & xy & xz \\ xy & \,\,\,y^2 & yz \\ xz & yz & \,\,z^2 \end{smallmatrix} \right] \rho(\vx) \, \textrm{d}\vx}
\label{eq:sigma}
\end{equation}
Through expansion of \eqref{eq:SigmaC} and using \eqref{eq:density1} and \eqref{eq:density2}, an analog to the parallel axis theorem can be verified as:
\begin{equation}
\boldsymbol{\Sigma} = \boldsymbol{\Sigma}_C +  m \vc \vc\T
\label{eq:parallelAxisSigma}
\end{equation}
As a key feature, using the relationship $\boldsymbol{\Sigma}= \tfrac{1}{2} \rm{Tr}(\vIbar) \bone_3 - \vIbar$ from \eqref{eq:SigmaFromIbar}, $\boldsymbol{\Sigma}$
%\[
%\boldsymbol{\Sigma}(\vpi) = \left[
%\begin{smallmatrix} \tfrac{1}{2} (I_{yy} + I_{zz} - I_{xx})  & -I_{xy} & I_{xz} \\
%							-I_{xy} & \tfrac{1}{2} (I_{xx} + I_{zz} - I_{yy }) & -I_{yz} \\
%							-I_{xz} & -I_{yz} & \tfrac{1}{2} ( I_{xx} + I_{yy} - I_{zz} ) \end{smallmatrix} \right]
%\]
is verified linear in the \revone{inertial} parameters $\vpi$. %The parallel axis theorem for $\boldsymbol{\Sigma}_C$ readily admits an LMI characterization through application of the Shur compliment lemma.

%The following proposition shows how this LMI can be expressed over $\vIbar$ directly, without transformation to the CoM.

\begin{definition}[Pseudo-Inertia Matrix]
The pseudo-inertia matrix $\pInertia \revone{\in \mathbb{R}^{4\times 4}}$ of a rigid body is defined by
\[
% \begin{bmatrix}   \tfrac{1}{2} \rm{Tr}(\vIbar) \bone_3 - \vIbar & \vh \\[.5ex] 
%\vh\T &  m \end{bmatrix} 
\pInertia = \begin{bmatrix}  \boldsymbol{\Sigma}\phantom{{}\T} & \vh \\[.5ex] 
\vh\T &  m \end{bmatrix} 
\]
%\[
%\pInertia(\vpi) = \int_{\mathbb{R}^3} \begin{bmatrix} \vx \\ 1 \end{bmatrix} \begin{bmatrix} \vx \\ 1 \end{bmatrix}^\top \rho(\vx) \, \rm{d} \vx 
%\]
%$\pInertia(\vpi)$ matches the form of a 1st-order moment matrix \cite{Lasserre08}.
\end{definition}
%Modulo row and column ordering, this matrix matches the content of a second-order moment matrix \cite{Lasserre01}.
\revone{Suppose some density $\rho(\cdot)$ with pseudo inertia $\pInertia$. The entries of $\pInertia$ are then given by definitions in \eqref{eq:density1}, \eqref{eq:density2}, and \eqref{eq:sigma}. These entries of $\pInertia$ include all moments of $\rho(\cdot)$, in the sense of Definition~\ref{def:moment}, up to second order. Additionally, given \revone{inertial} parameters $\vpi \in \mathbb{R}^{10}$, $\pInertia(\vpi)$ is verified linear in $\vpi$.}

\begin{theorem}[LMI Representation of $\mathcal{P}^*$]
\label{thm:main}
The set of physically-consistent parameters $\mathcal{P}^*$ for a single rigid body is strictly LMI representable. Its LMI representation is
\[
\mathcal{P}^* = \left\{ \vpi \in \mathbb{R}^{10} \,:\, \pInertia(\vpi) \succ 0
\right\}
\]
\end{theorem}

\begin{proof}
By the Schur complement lemma \cite[Section A.5.5]{BoydVanberghe04},  $\pInertia(\vpi) \succ 0$  if and only if $m(\vpi) > 0$ and ${\boldsymbol{\Sigma} - \tfrac{1}{m} \vh \vh\T \succ 0}$. By application of the parallel axis theorem for second moment matrices \eqref{eq:parallelAxisSigma}, and using $\vh = m \vc$, this is equivalent to ${\boldsymbol{\Sigma}_C \succ 0}$. From Proposition \ref{prop:cov}, this is equivalent to $\vIbar_C \succ 0$ and $\vIbar_C$ satisfies the triangle inequalities. Finally, from the main result of \cite{Traversaro16}, this is equivalent to $\vpi \in \mathcal{P}^*$. 
\end{proof}

%{\color{blue}
\begin{remark}
Again, taking a statistical perspective on the mass distribution, the results of Thm.~\ref{thm:main} can be seen to follow from a condition on the first and second moments \revone{in probability and statistics} \cite[Thm.~16.1.2]{Bertsimas00}. Suppose $\boldsymbol{\mu} \in \mathbb{R}^3$ and ${\mathbf E} \in \mathbb{S}^3$, then there exists a random variable $\vX \in \mathbb{R}^3$ such that $\boldsymbol{\mu} = E[\vX]$ and ${\mathbf E} = E[\vX \vX\T]$
if and only if
\[
\begin{bmatrix} {\mathbf E}\phantom{{}\T} & \boldsymbol{\mu} \\ \boldsymbol{\mu}\T & 1\end{bmatrix} \succeq 0
\] 
In comparison, the constraint $\pInertia(\vpi) \succ 0$ in Thm.~\ref{thm:main} is scaled by mass and enforces non-degeneracy of the distribution.
\end{remark}
%}

%{
%\color{blue}
%,Lee86 ,Jacak02
% and in early work on adaptive control (e.g. \cite{Neuman86})
The moment matrix $\pInertia(\vpi)$ was found commonly in early robot dynamics literature (e.g. \cite{bejczy1974robot,Uicker67}).  It has been employed in $4\times 4$ matrix forms of the dynamics for a rigid body \cite{Legnani96,Atchonouglo08b}. Using these $4\times4$ equations,  a parameter identification approach for a single rigid body was proposed in \cite{Atchonouglo08b}. The pseudo inertia was used to provide a left-invariant Riemannian metric over $\SE{3}$ \cite{Belta02} and appears in robotics books (e.g.~\cite{Yoshikawa03}). Despite its importance here, the pseudo inertia is notably lacking from current mainstream literature on robot dynamics.

It is interesting to note how the pseudo inertia $\pInertia(\vpi)$ compares to the standard spatial inertia $\vI(\vpi)$ in terms of the kinetic energy metric each provides. Suppose 
\begin{equation}
\mathbf{v} = \begin{bmatrix} \boldsymbol{\omega} \\ \boldsymbol{v} \end{bmatrix} \in \mathbb{R}^6 \textrm{~, and~}\mathbf{V} = \begin{bmatrix} \vS(\boldsymbol{\omega}) & \vel \\ \bzero &  0 \end{bmatrix} \in\se{3} \nonumber
\end{equation} 
It can be verified that the kinetic energy satisfies \cite{Yoshikawa03}
\begin{equation}
\tfrac{1}{2} \vv\T \vI(\vpi) \vv = \tfrac{1}{2} \textrm{Tr}\!\left( \mathbf{V}\, \pInertia(\vpi) \mathbf{V}\T  \right) \nonumber
\end{equation} 
Thus, while physical semi-consistency ensures that the associated kinetic energy metric is positive definite, additional constraints from the triangle inequalities enforce added structure on the metric. As has been shown for the first time here, the triangle inequalities are precisely what represent the difference between $\vI(\vpi)\succ0$ and $\pInertia(\vpi) \succ 0$. %This fundamental insight is one that is entirely new here.
%, enabled from the statistical perspective in this paper. 
                                                                                                                                                                                                                                                                                                                                                                                              %The pu
 The pseudo inertia is also of a lower dimension ($4\times4$) than the spatial inertia ($6\times6$). This provides computational benefits to enforcing LMIs on  $\pInertia(\vpi)$ in comparison to $\vI(\vpi)$.

\section{Contribution: LMI Constraints for \revone{Inertial} Parameter Realizability on Ellipsoids}
\label{sec:Extensions}
\newcommand{\vQ}{ {\mathbf{Q}}}

\revone{It has been shown that physical consistency of the \revone{inertial} parameters can be enforced with an LMI over all moments up to second order. The resultant LMIs have been explained through connections to the existence of probability measures. Even beyond mechanics and probability, however, the realizability of moment sequences represents a more general problem in mathematics. Namely, it represents the classical problem of moments (see e.g.~\cite{Fialkow10,Lasserre08}). With this insight, results are translated here to provide new conditions on density realizability with bounding-volume constraints. }

%, with recent advances in algebraic geometry enabling efficient numerical solution approaches \cite{Helton12}
%\subsection{Considering Bounding Volumes: The Ellipsoid Case}
Recent work has shown the benefits of including prior knowledge on the shape of each rigid body within parameter identification. Work by  Jovic et al.~\cite{Jovan16} for instance, enforced the CoM to reside within a bounding box estimated from CAD. New constraints are provided here that address second moments.

To begin, suppose a rigid body is known to reside within an ellipsoid $\mathcal{S} \subset \mathbb{R}^3$ described by:
\begin{equation}
\label{eq:bounding}
\mathcal{S}=\{ \vx \in \mathbb{R}^3 ~|~ (\vx - \vx_s)\T \vQ_s^{-1} (\vx - \vx_s) \le 1\}
\end{equation}
Before considering constraints on the second moments, it is noted that  the CoM constraint, $\vc(\vpi) \in \mathcal{S}$, can be formulated as an LMI over $\vpi$. When $m(\vpi)>0$, $\vc(\vpi) \in \mathcal{S}$ iff 
\begingroup\makeatletter\def\f@size{9}\check@mathfonts
\begin{equation}
\boldsymbol{C}(\vpi) =  \begin{bmatrix} m(\vpi) & \vh(\vpi)\T -m(\vpi) \vx_s\T  \\ \vh(\vpi) -m(\vpi) \vx_s & m(\vpi) \vQ_s \end{bmatrix} \succeq 0
\label{eq:comloc}
\end{equation}
\begingroup\makeatletter\def\f@size{10}\check@mathfonts
Again,  equivalence is due to the Schur complement lemma.

Yet, as the CoM approaches the edge of $\mathcal{S}$, large second moments would imply the existence of mass outside the ellipsoid. In fact, as the CoM of a rigid body approaches the edge of a bounding ellipsoid, the rigid body necessarily degenerates to a point mass.
%\footnote{When the CoM is at the edge, take the supporting hyperplane for the bounding ellipsoid that passes through the CoM. Equal mass must be on each side of this hyperplane. Thus, all the density must exist at the CoM.} 
Thus, constraints that $\vc(\vpi) \in \mathcal{S}$ alone are not sufficient for $\vpi$ to be $\mathcal{S}$-density realizable. Drawing on \revone{the classical problem of moments}, Theorem~4.7(b) from \cite{Fialkow10} can be translated as follows. 

\begin{theorem}[Density Realizability on an Ellipsoid]
Suppose a bounding ellipsoid $\mathcal{S} \subset \mathbb{R}^3$. Let $\vQ\in \mathbb{R}^{4\times4}$ such that
\begin{equation}
\mathcal{S} = \left \{ \vx \in \mathbb{R}^3 ~:~ \left[\begin{smallmatrix} \vx \\[.4ex]  1 \end{smallmatrix}\right]\T\! \vQ \left[\begin{smallmatrix} \vx \\[.4ex]  1 \end{smallmatrix} \right] \ge 0 \right\} \nonumber
\end{equation}
Then, $\vpi$ is $\mathcal{S}$-density realizable if and only if 
\begin{equation}
\label{eq:newCondition}
\pInertia(\vpi) \succ 0 \rm{~and~} \rm{Tr}\left( \pInertia(\vpi) \, \vQ \right) \ge 0\,.
\end{equation}
Further, any $\vpi$ is $\mathcal{S}$-density realizable iff it can be represented by four point masses $m_k$ at positions $\vx_k \in \mathcal{S}$. 
%such that:
%\begin{equation}
%\pInertia(\vpi) = \sum_{k=1}^{4} m_k \begin{bmatrix} \vx_k \\ 1 \end{bmatrix} \begin{bmatrix} \vx_k \\ 1 \end{bmatrix} \T \nonumber
%\end{equation}
\label{thm:RepresentabilityEllipse}
\end{theorem}
\begin{proof} See \cite[Theorem~4.7(b)]{Fialkow10}. \end{proof}

Note that by taking into account second moments, the condition $\rm{Tr}\!\left( \pInertia(\vpi)\,\vQ  \right) \ge 0$ in \eqref{eq:newCondition} is a 1D linear inequality constraint. The looser condition \eqref{eq:comloc} is a $4\times4$ LMI. Thus, the new conditions are both tighter, and computationally more efficient to enforce. Figure \ref{fig:EllipseExamples} shows a number of cases, illustrating both the role of the CoM location and the shape of the \revone{covariance ellipsoid} $\mathcal{E}_{\vpi}$ on $\mathcal{S}$-density realizability.  The following corollary is provided to accommodate more complex bounding shapes.

\begin{figure}
\center
\includegraphics[width =  \columnwidth]{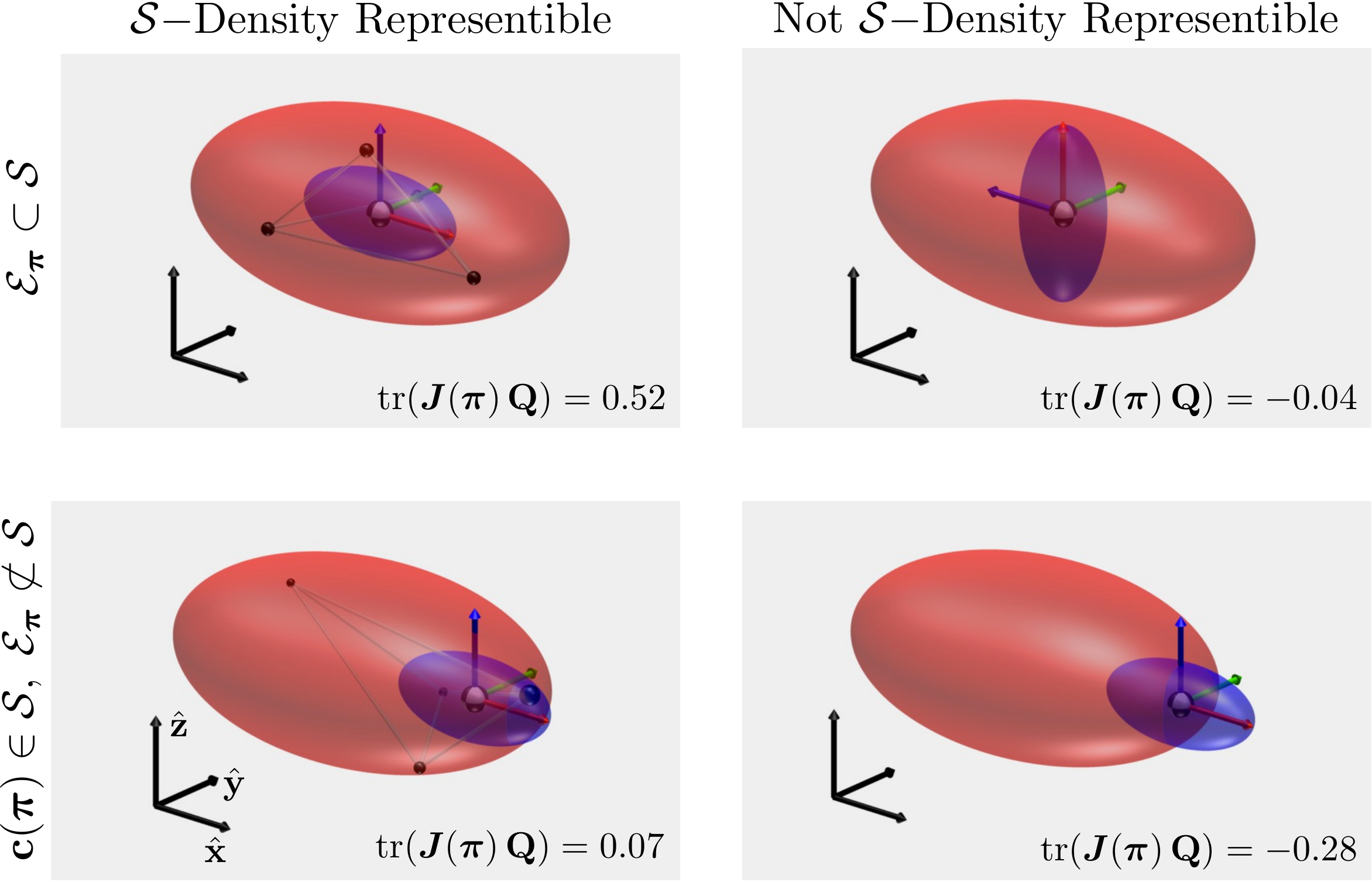}
\caption{Cases illustrating the role of the center and shape of the \revone{covariance ellipsoid} $\mathcal{E}_{\vpi}$ on $\mathcal{S}$-density realizability of $\vpi$. In each case, the ellipsoid $\mathcal{S}$ has semi-axes of length $\sqrt{5}, \sqrt{2},1$ in the $\hat{\vx}, \hat{\vy}, \hat{\vz}$ directions. $\mathcal{E}_{\vpi}$ has semi-axes ${\color{red}{\sqrt{0.9}}}, {\color{green} \sqrt{0.2}}, {\color{blue} \sqrt{0.2}}$ with principle axes colored accordingly in the figure. When a mass distribution on $\mathcal{S}$ exists, a distribution by four point masses is shown, as guaranteed to exist through Thm.~4.}
\label{fig:EllipseExamples}
\vspace{-12px}
\end{figure}

\begin{corollary}[Density Realizability on the Union of Ellipsoids] Suppose a rigid body is known to reside within the union of $n_\ell$ ellipsoids $\mathcal{S}=\cup_{j=1}^{n_\ell} \mathcal{S}_j$. Then, its \revone{inertial} parameters $\vpi$ are $\mathcal{S}$-density realizable if and only if there exist parameters $\{\vpi_{j}\}_{j=1}^{n_\ell}$ such that $\vpi = \sum_j \vpi_j$ and each $\vpi_j$ is $\mathcal{S}_j$-density realizable as verified by Thm.~\ref{thm:RepresentabilityEllipse}.
\end{corollary}

\begin{remark}
\label{remark:adaptation}
LMIs for the nested sets $\mathcal{P}^*_\mathcal{S} \subset \mathcal{P}^*_{\mathbb{R}^3} \subset \mathcal{P}$ have applicability for other problems. Tightened convex constraints can only increase the rate of convergence for online parameter estimation \cite{Ioannou12}, or more generally provide faster decrease of Lyapunov-like functions in adaptive control~\cite{Slotine86}. 
\q{volume}{\revone{For a system of bodies, the relative volume of the set $\mathcal{P}^*_{\mathcal{S}_{1}} \times \cdots \times  \mathcal{P}^*_{\mathcal{S}_{n_b}}$ versus $\mathcal{P}^*_{\mathbb{R}^3} \times \cdots \times \mathcal{P}^*_{\mathbb{R}^3}$ decreases exponentially with $n_b$ due to its product structure. Thus, it is expected that these benefits will increase for higher-DoF robots.}}
\end{remark}

\newcommand{\mysp} {}

\vspace{-1ex}
\section{Experimental Validation}

\label{sec:results}

\begin{figure*}
\center
\includegraphics[width = .9\textwidth]{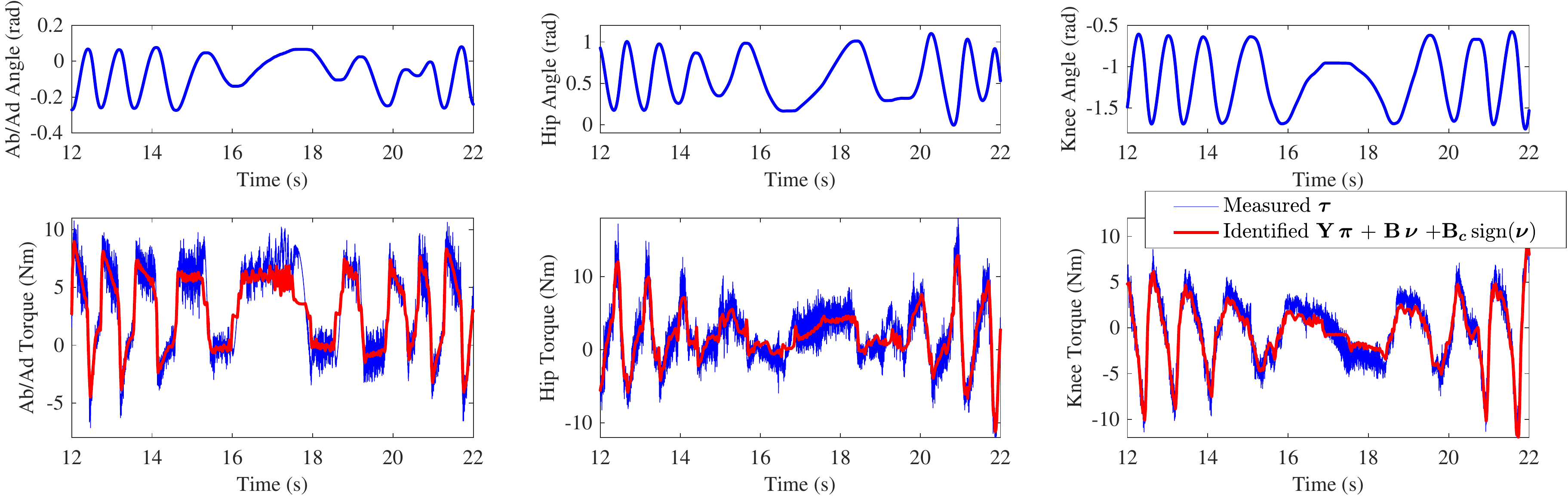}
\vspace{-6px}
\caption{Regression errors after identification of a leg for the MIT Cheetah 3. The data shown is distinct from that used in training. Motor torques are predicted with an overall RMS error of $1.46$ Nm. For comparison, the Coulomb friction was identified as $\mathbf{B}_c = {\rm diag}(3.12, 1.25, 0.95)$ Nm.}
\label{fig:tracking}
\vspace{-5px}
\end{figure*}

The proposed constraints were applied to identify a leg of the MIT Cheetah 3 robot. The robot, shown in Fig.~\ref{fig:leg}, has four 3-DoF legs where each DoF is driven by a proprioceptive actuator \cite{Wensing17b}. Each actuator includes a high-inertia rotor coupled to the joint by a $10.6:1$ gearbox. To  address actuator effects, the leg was treated as a system of $n_b=6$ bodies (3 links and 3 rotors). Joint rates at the gearbox output were used for~$\vnu \in\mathbb{R}^3$.  To account for transmission losses, \eqref{eq:regressor} was modified via diagonal matrices of viscous and Coulomb friction coefficients $\mathbf{B} \in \mathbb{R}^{3\times3}$ and $\mathbf{B}_c \in \mathbb{R}^{3\times3}$
\begin{equation}
\boldsymbol{\tau}  - {\mathbf B}\, \vnu -\mathbf{B}_c\,{\rm sign}(\vnu) = \vY(\vq, \vnu, \vnud)\, \vpi 
\label{eq:transmissionRegressor}
\end{equation}
Using Theorem \ref{thm:RepresentabilityEllipse}, physically realistic \revone{inertial} parameters can be identified alongside transmission effects.
\begingroup\makeatletter\def\f@size{9}\check@mathfonts
\begin{align}
\min_{\vpi,\mathbf{B}_c,\mathbf{B}}~ & \frac{1}{n_s}\sum_{m} \| \vY^{(m)}\, \vpi + {\mathbf B}\, \vnu^{(m)}\! + \mathbf{B}_c\,{\rm sign}(\vnu^{(m)})  -  \boldsymbol{\tau}^{(m)} \|^2~~~~~~~~~~~~~~~~~~~~~~~ \nonumber\\&~~+ w_\pi\|\vpi-\hat{\vpi}\|^2
\label{eq:optProb}\\
			\textrm{s.t.}~ 
						  & \boldsymbol{C}_i(\vpi_i) \succeq 0 \nonumber\\ %\hspace{40px}{\rm(CoM~within~an~Ellipsoid)} \nonumber\\
						  & \phantom{{}_i\,}\pInertia(\vpi_i) \succ 0 \nonumber \\ %\hspace{40px} {\rm(Density~Realizability)} \nonumber\\
						  & \mbox{\rlap{${\rm Tr}( \pInertia(\vpi_i) \, \vQ_i)\ge 0$}\hphantom{$\boldsymbol{C}_i(\vpi_i) \succeq 0$}}  \hspace{40px} \forall i \in \{1,\ldots,n_b\} \nonumber %{\rm(Density~on~an~Ellipsoid)} \nonumber%\\	
		%& \quad \quad \forall i \in \{1,\ldots,n_b\} \nonumber
\end{align}  \begingroup\makeatletter\def\f@size{10}\check@mathfonts
A small regularizing term $w_\pi\|\vpi-\hat{\vpi}\|^2$ was added where $\hat{\vpi}\in \mathbb{R}^{10 n_b}$ is a set of estimated \revone{inertial} parameters from CAD or other sources. Such regularization is common practice~\cite{Sousa14}. Bounding-ellipsoid parameters for the CoM and body, in $\boldsymbol{C}_i$ and $\vQ_i$ respectively, were set using geometry from CAD.
% By keeping $w_{\pi}$ small, it effectively assures that, within the nullspace of identification error, parameters are found with the closest match to estimates $\hat{\vpi}$. 

Data was gathered from a leg swinging experiment shown in the supplementary video. The leg was placed in a Cartesian impedance control mode, and the foot endpoint was commanded to move on a virtual ellipsoidal shell. The target point on the shell was set through spherical angles $(\phi,\theta)$, using rates $\dot{\phi} = A_\phi \sin(\omega_\phi t )$ and $\dot{\theta} = A_\theta \sin(\omega_\theta t)$ with $A_\phi = 12$ rad/s,  $A_\theta = 3.4$ rad/s, $\omega_\phi=1.63$ rad/s, and $\omega_\theta = 0.265$ rad/s. Data was sampled at 1 kHz. Joint actuator torques $\boldsymbol{\tau}_j$ were estimated from the torques commanded to motor drivers \cite{Wensing17b}. \q{RNEA}{\revone{At each data point, a modified version of the Recursive-Newton-Euler Algorithm \cite{Featherstone08} was used to compute each column of $\vY$.}}  Following this computation, the problem \eqref{eq:optProb} was solved with 10,000 samples using MOSEK  \cite{AndersonE03} in {\sc Matlab}. The problem took 1.67s to solve to global optimality on a 2011 Intel Core i5 MacBook Pro. Regularization of $w_\pi = 10^{-6}$ was used. 
\q{constraintsSatisfied}{\revone{It was verified that all physical-consistency constraints were satisfied. Numeric values for the identified parameters are provided in the supplementary material.}}

After identification, the RMS errors from \eqref{eq:transmissionRegressor} with the {\em next} 10,000 samples in the dataset were $1.48$, $1.69$, and $1.16$~Nm on the ab/ad, hip, and knee respectively. For comparison, Coulomb friction was found as $\mathbf{B}_c = {\rm diag}(3.12, 1.25, 0.95)$ Nm. Figure \ref{fig:tracking} shows validation results. 
\q{tauMeasured}{\revone{The measured $\btau$, used as input to the algorithm, shows high-frequency noise. This noise is from finite-differenced encoder signals used in online feedback. In contrast, the estimated data used non-causally filtered signals to compute $\vY$ offline. This results in a comparatively smoother identified estimate.}} 
\q{frictionTracking}{\revone{The estimation is notably poor on the ab/ad joint at 17.5s. However, this occurs at a time when the joint is not moving, and the sign of the Coulomb friction torque cannot be reliably predicted.}}

\revone{Figure \ref{fig:breakdown} shows an accounting of friction and \revone{inertial} effects on the identified model. Effects that dominate the required torques are shown to be dependent on both the motion as well as the joint. This demonstrates the need to identify all model components treated here. It is important to note that these effects will be robot and transmission specific.}

Figure~\ref{fig:generalization} shows the validation error with different training sample sizes $n_s$ and different constraints. As in Remark~\ref{remark:adaptation}, tighter constraints empirically result in an accurate model more rapidly as samples are added. \revone{Lower validation error for tighter constraints further demonstrates reduced overfitting.} These benefits are in addition to the fact that tighter constraints result in physically realistic model parameters. %Note that due to the product structure of the constraint sets for the system inertial parameters $\$

\begin{remark}
\vspace{-3px}
\q{scaleRemark}{\revone{
Although these results have only addressed a 3-DoF leg, physical-consistency constraints  are equally applicable to more complex rigid-body systems. Previous work demonstrated the scalability of least-squares identification to high-DoF platforms \cite{Ayusawa14,Jovan16}. The $4\times4$ LMIs proposed would present minor added overhead to the solution of these problems. As in Remark \ref{remark:adaptation}, benefits from our tightened constraints are expected to increase with the number of bodies.}}\vspace{-3px}
\end{remark}

\begin{figure}
\center
\includegraphics[width = \columnwidth]{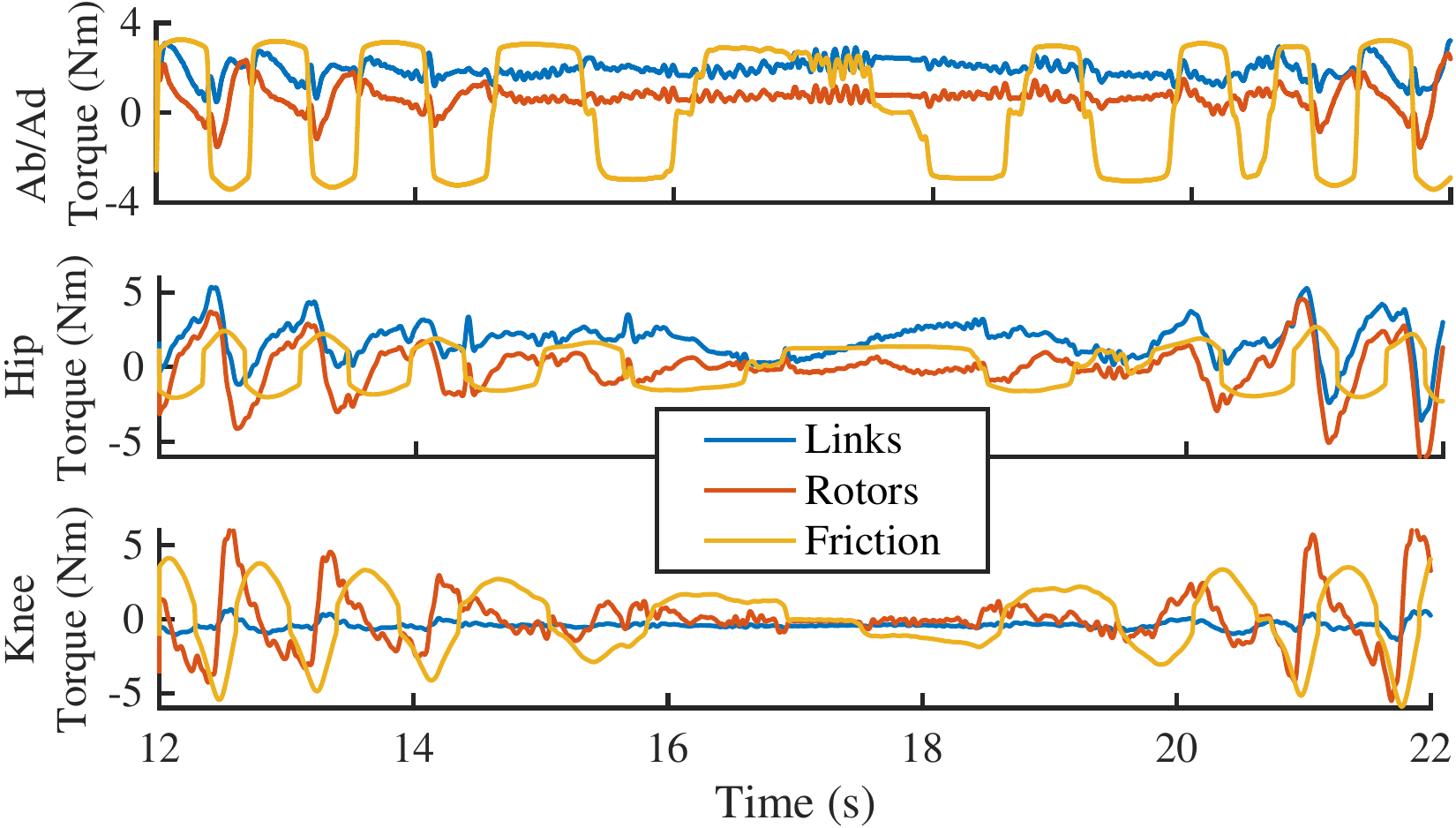}\vspace{-6px}
\caption{\revone{Identified contributions of the effects from link inertias, rotor inertias, and friction. Data is from the same experiment as Figure \ref{fig:tracking}.}}
\label{fig:breakdown}
\vspace{-3px}
\end{figure}

\begin{figure}
\center
\includegraphics[width = \columnwidth]{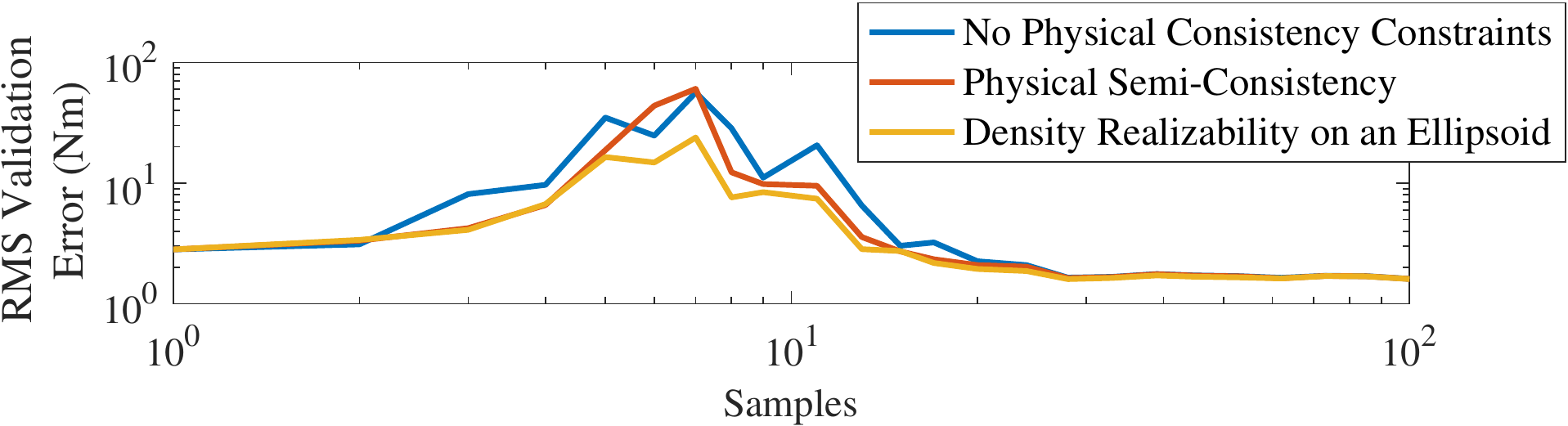}\vspace{-6px}
\caption{Validation error (after identification) versus training sample size.}
\label{fig:generalization}
\vspace{-10px}
\end{figure}

%Even given a $n_s$ samples, the stacked regressors
%\[
%\vY^{(1:n_s)} = \begin{bmatrix} \vY^{(1)\top} & \cdots &  \vY^{(n_s)\top} \end{bmatrix}\T
%\]
%can be verified to have a non-trivial minimal null space for any sequence of measurements. This motivated the study of the base parameters \cite{Atkeson86} to reduce the size of the unknown parameters. Yet, when attempting to enforce density realizability of each $\vpi_i$, a reduction to base parameters is not well motivated. In practice, regularization $w_\pi$ can be chosen to avoid numerical precision issues that may artificially decrease the dimension of the regressor null space beyond its theoretical minimum. The inclusion of regularization makes the above problem better conditioned, akin to a damped pseudoinverse.  

\vspace{-8px}
\section{Conclusions}
\vspace{-0px}
\label{sec:conclusions}
	
	\revone{This paper has introduced LMIs to rigorously characterize physical consistency of rigid-body \revone{inertial} parameters. Rather than focusing on the moments of inertia, physical plausibility is directly assessed using the types of moments encountered in probability and statistics. With this observation, LMIs involving a pseudo-inertia matrix and its density-weighted covariance have been shown to  tightly characterize physical consistency.	
%	that arise in addressing this issue are closely connected to constraints which the plausibility of moment sequences in probability and statistics. Both of these problems fall under the classical problem of moments in mathematics.  
%	\revone{This paper has advocated for taking a statistical perspective on the mass distribution of a rigid body when approaching physical-consistency constraints for \revone{inertial} parameter identification. Rather than focus on moments of inertia, physical plausibility is directly assessed through considering the types of second moments encountered in probability and statistics. With this observation, LMIs involving the pseudo-inertia matrix $\pInertia(\vpi)$ have been shown to efficiently enforce density realizability constraints. %Development illustrated that the difference between positive definiteness of the standard 6D inertia tensor $\vI(\vpi)$, and the 4D pseudo-inertia matrix is precisely the inclusion of the triangle inequalities embedded within $\pInertia(\vpi)\succ 0$.
	%Through the form of $\pInertia(\vpi)$, it was argued that constraints on the \revone{inertial} parameters might instead be viewed as a special case of the classical problem of moments. 
	LMI constraints for density realizability on an ellipsoid were also proposed through connections to the classical problem of moments in mathematics.}
	%This result provides a rich set of constraints that can be used to solve physically consistent \revone{inertial} parameter identification problems to guaranteed  global optimality. 

	\q{closing}{\revone{In closing, recall that while intuition was borrowed from probability and statistics, there is nothing stochastic about our results. Given a bounding ellipsoid for a rigid body, the new constraints characterize a cone where its \revone{inertial} parameters must reside with certainty. Beyond least-squares considerations, a quantified treatment of stochasticity can play an important role in inferring models from uncertain measurements, with a rich literature on the topic (e.g.~\cite{Swevers97,Ting06,Janot14,Calafiore00,Poignet05}). The incorporation of rigorous physical-consistency considerations into this inference represents an interesting next step for \revone{inertial} parameter identification.}}

\bibliographystyle{IEEEtran}
\bibliography{InertialLMIs}

%\bibliography{/Users/pwensing/mitdropbox/Research/Papers/References}

\end{document}

%% file: preamble.tex
 \usepackage{algorithmic}
\usepackage{amsmath}
\usepackage{amssymb}
\usepackage{amstext}

\usepackage{color}
\usepackage{url}

\renewcommand{\vec}[1]{ {\mathbf{#1}}}
\usepackage{multirow}
\newcommand{\btau}   {\mbox{\boldmath $\tau$}}

\newcommand{\vnu} {\boldsymbol{\nu}}
\newcommand{\vnud} {\dot{\boldsymbol{\nu}}}

\newcommand{\T}{^\top}

\newcommand \se     {\mathrm{se}}
\newcommand \SE     {\mathrm{SE}}

\newcommand{\bzero}{\mbox{\boldmath $0$}}

\usepackage{mdframed}

{\end{tabbing} \end{mdframed} \vspace{5px}}

\newcommand{\vc}{\vec{c}}
\newcommand{\vg}{\vec{g}}

\newcommand{\vI}{{\mbox{\boldmath$\mathit{I}$}}}
\newcommand{\pInertia}{{\mbox{\boldmath$\mathit{J}$}}}

\newcommand{\BM}{\begin{bmatrix}}
\newcommand{\EM}{\end{bmatrix}}
\newcommand{\bone}{{\bf 1}}

\newcommand{\vA}{\vec{A}}

\newcommand{\vB}{\vec{B}}

\newcommand{\vC}{\vec{C}}

\newcommand{\vE}{\vec{E}}
\newcommand{\vf}{\vec{f}}

\newcommand{\vh}{\vec{h}}

\newcommand{\vH}{\vec{H}}

\newcommand{\vIbar}{\,\bar{\!\vI}}
\newcommand{\vJ}{\vec{J}}

\newcommand{\vY}{\vec{Y}}
\newcommand{\vq}{\vec{q}}
\newcommand{\vpi}{\boldsymbol{\pi}}

\newcommand{\vR}{\vec{R}}

\newcommand{\vS}{\vec{S}}

\newcommand{\vv}{\vec{v}}

\newcommand{\vX}{\vec{X}}

\newcommand{\vx}{\vec{x}}
\newcommand{\vy}{\vec{y}}
\newcommand{\vz}{\vec{z}}

\newcommand{\beq}{\begin{equation}}
\newcommand{\eeq}{\end{equation} }